\documentclass{ecai}
\usepackage{times}
\usepackage{graphicx,xcolor}
\usepackage{latexsym}
\usepackage{amsmath,amsfonts,amsthm,amssymb}
\usepackage{subfigure}
\usepackage{algorithm,algorithmic}

\newtheorem{thm}{Theorem}
\newtheorem{proposition}{Proposition}
\newtheorem{lemma}{Lemma}

\begin{document}

\title{Coulomb Autoencoders}

\author{Emanuele Sansone\institute{Huawei Noah's Ark Lab, UK , London, email: emanuele.sansone@huawei.com} \and Hafiz Tiomoko Ali\institute{Huawei Noah's Ark Lab, UK , London} \and Jiacheng Sun\institute{Huawei Noah's Ark Lab, China, Shenzhen}}

\maketitle

\begin{abstract}
Learning the true density in high-dimensional feature spaces is a well-known problem in machine learning. In this work, we consider generative autoencoders based on maximum-mean discrepancy (MMD) and provide theoretical insights. In particular, (i) we prove that MMD coupled with Coulomb kernels has optimal convergence properties, which are similar to convex functionals, thus improving the training of autoencoders, and (ii) we provide a probabilistic bound on the generalization performance, highlighting some fundamental conditions to achieve better generalization.
We validate the theory on synthetic examples and on the popular dataset of celebrities’ faces, showing that our model, called Coulomb autoencoders, outperform the state-of-the-art.
\end{abstract}

\allowdisplaybreaks

{\section{Introduction}
\label{sec:intro}}
Deep generative models, like generative adversarial networks (GANs) and autoencoder-based models, represent the most promising research directions to learn the underlying density of data. Each of these families have their own limitations. On one hand, generative adversarial networks are difficult to train due to the mini-max nature of the optimization problem. On the other hand, autoencoder-based models, while more stable to train, often produce samples of lower quality compared to GANs. In this work, we attempt to address the issues of generative autoencoders.

Learning the unkwnown density in autoencoders requires to minimize two terms, namely the error between the input data and their reconstructed version, together with a distance between a prior density and the density induced by the encoder function. Note that by choosing different distances, we obtain different families of autoencoders. For example, when using the Kullback-Leibler divergence (KL), the corresponding models are variational autoencoders (VAEs)~\cite{kingma2014vae,rezende2014stochastic}, while when choosing the maximum-mean discrepancy (MMD), we obtain Wasserstein autoencoders (WAEs)~\cite{tolstikhin2018wasserstein}.
The main advantage of WAEs over VAEs is that MMD allows using encoders with deterministic outputs, while, by definition, KL requires only encoders with stochastic outputs. In fact, the stochastic encoder in VAEs is driven to produce latent representations that can be similar among different input samples, thus generating conflicts during reconstruction~\cite{tolstikhin2018wasserstein}, while the deterministic encoder in WAEs is driven to learn latent representations that are different for different input samples. Therefore, MMD should be preferred over KL, when using deterministic encoders. 
This work focuses on MMD-based autoencoders and provides two theoretical insights. Regarding the first contribution, we study the critical points of MMD coupled with Coulomb kernels and show that all local extrema are global and that the set of saddle points has zero Lebesgue measure. This result is particularly interesting from the optimization perspective, as MMD coupled with Coulomb kernels can be maximized/minimized through local-search algorithms (like gradient descent), without being trapped into local minima or saddle points. In the context of autoencoders, using MMD with Coulomb kernels allows to mitigate the problem of local minima and achieve better generalization performance, as validated through experiments on synthetic and real-world datasets. Regarding the second contribution, we provide a probabilistic bound on the generalization performance for MMD-autoencoders, highlighting the fact that the reconstruction error is crucial to achieve better generalization and that architecture design is the most important aspect to control it.

The rest of the paper is organized as follows. In Section~\ref{sec:theory}, we provide the two theoretical results. In Section~\ref{sec:related}, we review the literature of recent generative models. Finally we discuss the experiments in Section~\ref{sec:experiments}.

{\section{Formulation and Theoretical Analysis}\label{sec:theory}}
This section deals with the problem of density estimation. The goal is to estimate 
the unknown density function $p_\mathbf{X}(\mathbf{x})$, whose support is defined by 
$\Omega_{\mathbf{x}}\subset\mathbb{R}^d$. 


We consider two continuous functions $f:\Omega_{\mathbf{x}}\rightarrow\Omega_{\mathbf{z}}$ and 
$g:\Omega_{\mathbf{z}}\rightarrow\Omega_{\mathbf{x}}$, where
$\Omega_{\mathbf{z}}\subseteq\mathbb{R}^h$ with $h$ equal to the intrinsic dimensionality
of $\Omega_{\mathbf{x}}$. Furthermore, we consider that $g(f(\mathbf{x}))=\mathbf{x}$ for 
every $\mathbf{x}\in\Omega_{\mathbf{x}}$, namely that $g$ is the left inverse for $f$ on
domain $\Omega_{\mathbf{x}}$. In this work, $f$ and $g$ are neural networks parameterized
by vectors $\boldsymbol{\theta}$ and $\boldsymbol{\gamma}$, respectively. $f$ is called the
encoding function, taking a random input $\mathbf{x}$ with density $p_{\mathbf{X}}(\mathbf{x})$
and producing a random vector $\mathbf{z}$ with density $q_{\mathbf{Z}}(\mathbf{z})$, while
$g$ is the decoding function taking $\mathbf{z}$ as input and producing the random vector 
$\mathbf{y}$ distributed according to $q_{\mathbf{Y}}(\mathbf{y})$. Note that, 
$p_{\mathbf{X}}(\mathbf{x})=q_{\mathbf{Y}}(\mathbf{y})$, since $\mathbf{y}=g(\mathbf{z})=
g(f(\mathbf{x}))=\mathbf{x}$ for every $\mathbf{x}\in\Omega_{\mathbf{x}}$. This is already
a density estimator, but it has the drawback that in general $q_{\mathbf{Z}}(\mathbf{z})$
cannot be written in closed form. Now, define 
$p_{\mathbf{Z}}(\mathbf{z})$ an arbitrary density with support $\Omega_{\mathbf{z}}$, that
has a closed form.\footnote{In this work we consider $p_{\mathbf{Z}}(\mathbf{z})$ as a
standard multivariate Gaussian density.} 
Our goal is to guarantee that $q_{\mathbf{Z}}(\mathbf{z})=p_{\mathbf{Z}}(\mathbf{z})$
on the whole support, while maintaining $g(f(\mathbf{x}))=\mathbf{x}$ for every 
$\mathbf{x}\in\Omega_{\mathbf{x}}$. This allows us to use the decoding function as a generator
and produce samples distributed according to $p_{\mathbf{X}}(\mathbf{x})$. Therefore, the problem 
of density estimation in a high-dimensional feature space is converted into a problem of estimation 
in a lower dimensional vector space, thus overcoming the curse of dimensionality.

The objective of our minimization problem is defined as follows:
\begin{align}
    \mathcal{L}(f,g)=&\int_{\Omega_{\mathbf{x}}}\|\mathbf{x}-g(f(\mathbf{x}))\|^2
    p_\mathbf{X}(\mathbf{x})d\mathbf{x} \nonumber\\
    &+\lambda\int_{\Omega_{\mathbf{z}}}\int_{\Omega_{\mathbf{z}}}
    \phi(\mathbf{z})\phi(\mathbf{z}')k(\mathbf{z},\mathbf{z}')d\mathbf{z}d\mathbf{z}'
    \label{eq:objective}
\end{align}
where $\phi(\mathbf{z}){=}p_{\mathbf{Z}}(\mathbf{z}){-}q_{\mathbf{Z}}(\mathbf{z})$,
$k(\cdot,\cdot)$ is a kernel function and $\lambda$ is a positive scalar hyperparameter weighting the two addends.
Note that the first term in~(\ref{eq:objective}) reaches its global minimum
when the encoding and the decoding functions are invertible 
on support $\Omega_{\mathbf{x}}$, while
the second term in~(\ref{eq:objective}) is globally optimal  
when $q_{\mathbf{Z}}(\mathbf{z})$ equals $p_{\mathbf{Z}}(\mathbf{z})$
(see the supplementary material for a recall of its properties). 
Therefore, the global minimum of~(\ref{eq:objective}) satisfies our initial 
requirements and the optimal solution corresponds to the case 
where $q_{\mathbf{Y}}(\mathbf{y})=p_\mathbf{X}(\mathbf{x})$.

\begin{figure}[t!]
    \centering
        \subfigure[Gaussian]{\includegraphics[width=0.23\linewidth]{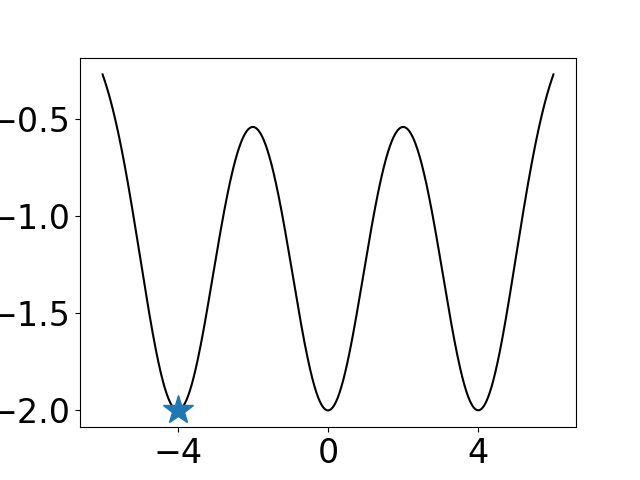}}
        \subfigure[IMQ]{\includegraphics[width=0.23\linewidth]{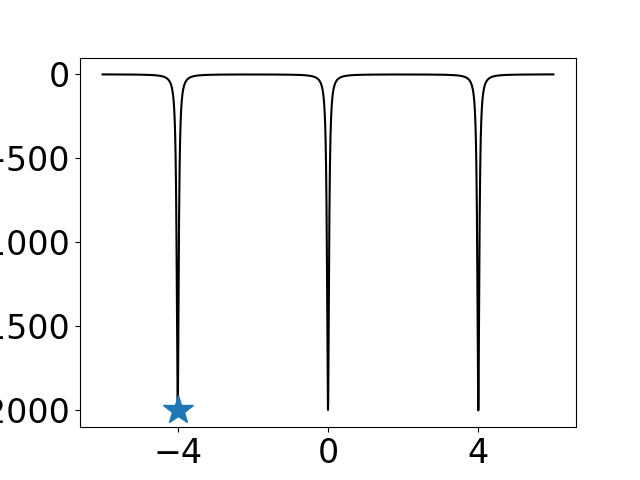}}
        \subfigure[Coulomb]{\includegraphics[width=0.23\linewidth]{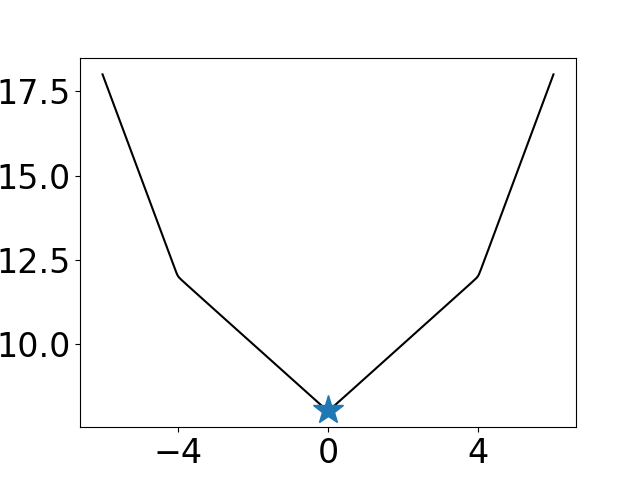}}
        \subfigure[Solution]{\includegraphics[width=0.23\linewidth]{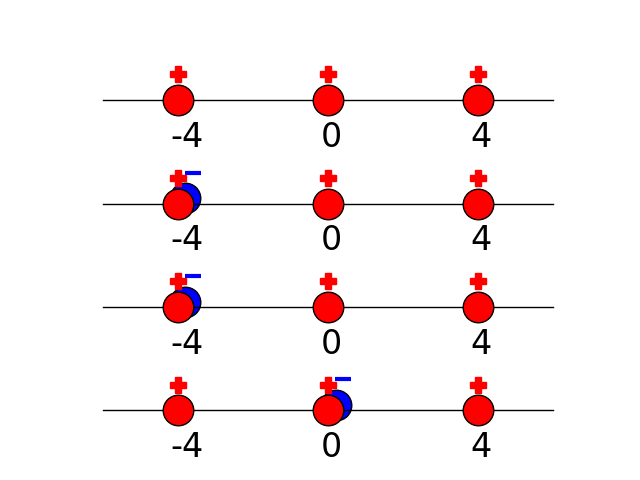}}\\
        \subfigure[Gaussian]{\includegraphics[width=0.23\linewidth]{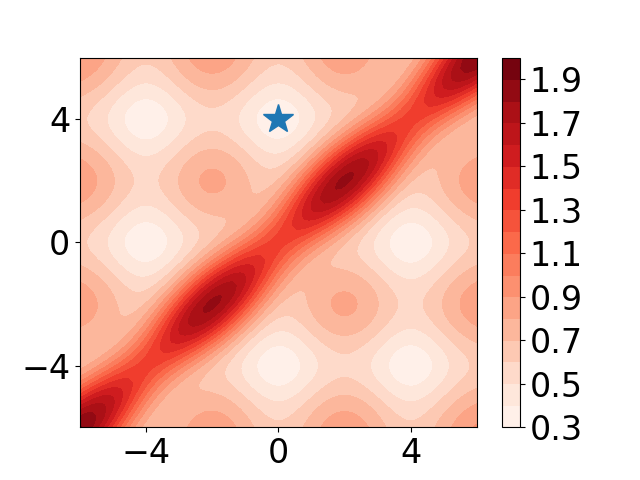}}
        \subfigure[IMQ]{\includegraphics[width=0.23\linewidth]{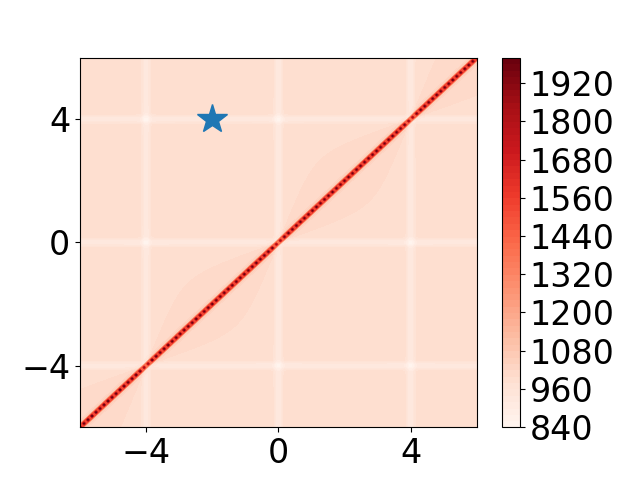}}
        \subfigure[Coulomb]{\includegraphics[width=0.23\linewidth]{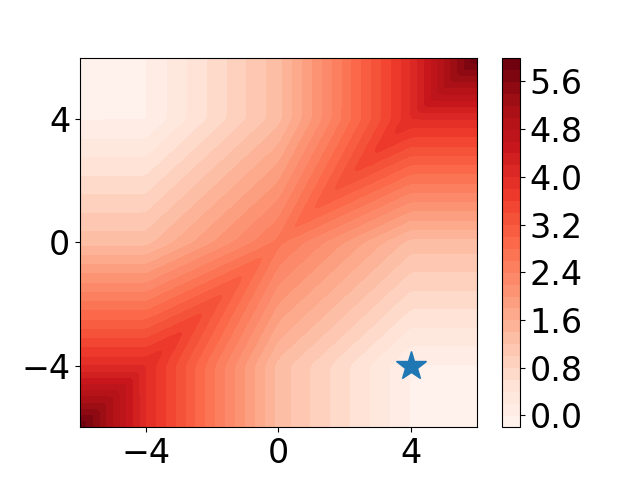}}
        \subfigure[Solution]{\includegraphics[width=0.23\linewidth]{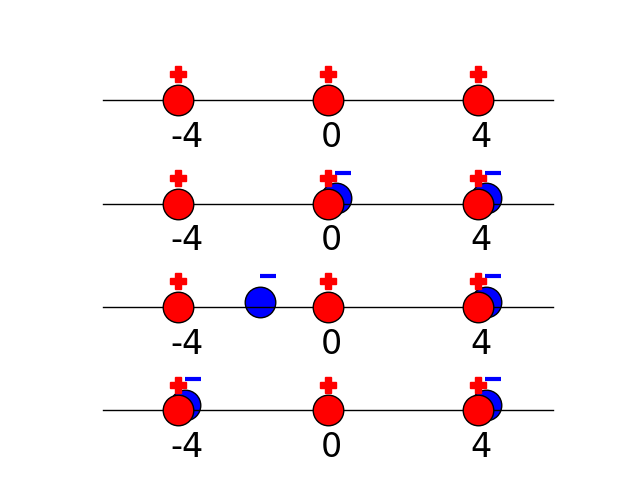}}
    \caption{Monodimensional cases with single ((a)-(d)) and pair of negative charged particles ((e)-(h))). (a-c) and (e-g) are the plots of the regularizer in~(\ref{eq:objective}) over different locations of the negative particles for the Gaussian, the inverse quadratic and the Coulomb kernels, respectively. (d) and (h) show possible minimizers (for the respective kernels).}
    \label{fig:demo}
\end{figure}

\subsection{Convergence properties}
The integrals in~(\ref{eq:objective}) cannot be
computed exactly since $p_\mathbf{X}(\mathbf{x})$ is unknown and $q_{\mathbf{Z}}(\mathbf{z})$
is not defined explicitly. As a consequence, we use the unbiased estimate
of~(\ref{eq:objective}) as a surrogate for optimization, namely:
\begin{align}
    \hat{\mathcal{L}}(f,g){=}& 
    \sum_{\mathbf{x}_i{\in}\mathcal{D}_\mathbf{x}}\frac{\|\mathbf{x}_i{-}g(f(\mathbf{x}_i))\|^2}{N} 
    +\lambda \Bigg\{\frac{1}{N(N{-}1)}\sum_{\substack{\mathbf{z}_i,\mathbf{z}_j{\in}\mathcal{D}_\mathbf{z}\\j\neq i}}
    k_{i,j} \nonumber\\ 
    &-\frac{2}{N^2}\sum_{\mathbf{z}_i{\in}\mathcal{D}_\mathbf{z}}
    \sum_{\substack{\mathbf{z}_j{\in}\mathcal{D}_\mathbf{z}^f}}k_{i,j}
    +\frac{1}{N(N{-}1)}\sum_{\substack{\mathbf{z}_i,\mathbf{z}_j{\in}\mathcal{D}_\mathbf{z}^f\\j\neq i}}
    k_{i,j}\Bigg\}
    \label{eq:empobjective}
\end{align}
where $k_{i,j}\doteq k(\mathbf{z}_i,\mathbf{z}_j)$ and 
$\mathcal{D}_\mathbf{x}{=}\{\mathbf{x}_i\}_{i=1}^{N}$,
$\mathcal{D}_\mathbf{z}{=}\{\mathbf{z}_i\}_{i=1}^{N}$ and
$\mathcal{D}_\mathbf{z}^f{=}\{f(\mathbf{x}_i)\}_{i=1}^{N}$ are 
three finite set of samples drawn from $p_\mathbf{X}(\mathbf{x})$, 
$p_{\mathbf{Z}}(\mathbf{z})$ and $q_{\mathbf{Z}}(\mathbf{z})$, respectively.

Note that the MMD term, corresponding to the last three addends in~(\ref{eq:empobjective}), is not convex in the set of unknowns $\mathcal{D}_\mathbf{z}^f$. This means that it is not possible in general to ensure convergence to the global minimum. Nevertheless, we can prove that, for a specific family of kernels, called Coulomb kernels~\cite{hochreiter2005optimal,unterthiner2017coulomb}, this property can be achieved. In fact,
\begin{thm}
\label{th:mmd}
Assume that 
\begin{enumerate}
	\item $N > h$.
	\item $\forall \mathbf{z}_i, \mathbf{z}_j \in \mathcal{D}_\mathbf{z}, $ $\mathbf{z}_i \neq \mathbf{z}_j$
	\item The kernel function satisfies the Poisson's equation, namely $-\nabla_\mathbf{z}^2k(\mathbf{z},\mathbf{z}')=\lambda\delta(\mathbf{z}-\mathbf{z}'),\quad 
    \forall \mathbf{z},\mathbf{z}'\in\mathbb{R}^h$. And the solution can be written in the following form
    \begin{equation}
        k(\mathbf{z},\mathbf{z}')=\left\{\begin{array}{ll}
            -\frac{\lambda}{2\pi}\ln\|\mathbf{z}-\mathbf{z}'\| & h=2 \\
            \frac{\lambda}{\beta\mathcal{S}_{h}\|\mathbf{z}-\mathbf{z}'\|^\beta} & \beta=h-2, h>2
        \end{array}\right.
        \label{eq:coulomb}
    \end{equation}
    where $\mathcal{S}_{h}$ is the surface area of a $h$-dimensional unit ball.
\end{enumerate}
Then, the MMD term in~(\ref{eq:empobjective}) satisfies the following general properties:
\begin{enumerate}
	\item all local extrema are global.
	\item the set of saddle points have zero Lebesgue measure.
\end{enumerate}
Furthermore, the set of all global minima is finite and consists of all
possible permutations of the elements in $\mathcal{D}_\mathbf{z}$. In other words,  $\mathcal{D}_\mathbf{z}^f=\mathcal{D}_\mathbf{z}$.
\end{thm}
\begin{proof}
Define the MMD term as:
\begin{align}
	&\hat{\mathcal{G}}(\{\mathbf{x}_i\}_{i=1}^N,\{\mathbf{z}_i\}_{i=1}^N)\\ \nonumber
	&=\frac{1}{N(N-1)}\sum_{i=1}^N\sum_{j\neq i}k(\mathbf{x}_i,\mathbf{x}_j)+\frac{1}{N(N-1)}\sum_{i=1}^N\sum_{j\neq i}k(\mathbf{z}_i,\mathbf{z}_j)\\ \nonumber
	&-\frac{2}{N^2}\sum_{i=1}^N\sum_{j=1}^Nk(\mathbf{x}_i,\mathbf{z}_j)
\end{align}
By the definition of Poisson kernel, we get the Laplacian of $\hat{\mathcal{G}}$ as
\begin{align}
	\nabla_{\mathbf{x}_i}^2\hat{\mathcal{G}}=-\frac{2}{N(N-1)}\sum_{j\neq i}\delta(\mathbf{x}_i-\mathbf{x}_j)+\frac{2}{N^2}\sum_{j=1}^N\delta(\mathbf{x}_i-\mathbf{z}_j)
\end{align}
Thus, $\hat{\mathcal{G}}$ is harmonic except on the set $H=\{\mathbf{x}:\mathbf{x}_i \neq \mathbf{x}_j, \mathbf{x}_i \neq \mathbf{z}_j, \forall i,j =1,...,N\}$. By the Maximal Principle of harmonic functions, $\hat{\mathcal{G}}$ has no local extrema and all the extrema are global.\\
On the other hand, the saddle points of $\hat{\mathcal{G}}$ satisfy
\begin{align}
	\nabla_{\mathbf{x}_i}\hat{\mathcal{G}}=&-\frac{1}{N(N-1)}\sum_{j\neq i}\frac{\mathbf{x}_i-\mathbf{x}_j}{\|\mathbf{x}_i-\mathbf{x}_j\|^h}+ \nonumber\\
	&\frac{1}{N^2}\sum_{j=1}^N\frac{\mathbf{x}_i-\mathbf{z}_j}{\|\mathbf{x}_i-\mathbf{z}_j\|^h}=0
\end{align}
This implies that
\begin{align}
	F(\mathbf{x_i})\triangleq\sum_{j=1}^N\frac{\mathbf{x}_i-\mathbf{z}_j}{\|\mathbf{x}_i-\mathbf{z}_j\|^h}=\frac{N}{N-1}\sum_{j\neq i}\frac{\mathbf{x}_i-\mathbf{x}_j}{\|\mathbf{x}_i-\mathbf{x}_j\|^h}
\end{align}
As $\hat{\mathcal{G}}$ is harmonic except on the set $H$, $F$ is analytic except on $H$. By the symmetry of the right hand side, we have $\sum_{i=1}^{N}F(\mathbf{x}_i)=0$. Define $$A=\{(\mathbf{x}_1,...,\mathbf{x}_N):\mathbf{x}_i\in\mathbb{R}^h\backslash\{\mathbf{z}_i\}_{i=1}^N,\sum_{i=1}^{N}F(\mathbf{x}_i)=0\}$$
If $(\mathbf{x}_1,...,\mathbf{x}_N)\in A$, then $\mathbf{x}_N\in F^{-1}(-\sum_{i=1}^{N-1}F(\mathbf{x}_i))$.
By considering that $F$ is analytic and nonconstant and by using the Fubini Theorem we get that the measure of $A$ is
\begin{align}
	|A|&=\int_{(\mathbb{R}^h)^{N}}\chi_Ad\mathbf{x}_1...d\mathbf{x}_{N}\nonumber \\ 
	&=\int_{(\mathbb{R}^h)^{N-1}}(\int_{\mathbb{R}^h}\chi_Ad\mathbf{x}_{N})d\mathbf{x}_1...d\mathbf{x}_{N-1}\nonumber \\ 
	&=\int_{(\mathbb{R}^h)^{N-1}}|F^{-1}(-\sum_{i=1}^{N-1}F(\mathbf{x}_i))|d\mathbf{x}_1...d\mathbf{x}_{N-1}=0
	\label{eq:measure}
\end{align}
where $\chi_A$ is the characteristic function, equals 1 on $A$, otherwise equals 0 and $|\cdot|$ denotes the Lebesgue measure operator. The third equality in~(\ref{eq:measure}) holds because we know that for a nonconstant analytic function, its inverse image at a value is of zero measure (w.r.t h-dimensional Lebesgue measure). As the saddle point is a subset of $A$, so the set of saddle points have zero Lebesgue measure.
\end{proof}

The assumptions of the theorem are quite general in practice. In fact, the requirement $N>h$ is generally valid in applications involving autoencoders. The second assumption is valid with probability 1, as long as the elements in $\mathcal{D}_\mathbf{z}$ are drawn independently from $p_Z$.

Note that, since the set of saddle points has zero measure, optimization through local search methods can converge to global minima. This is an important characteristic which is similar to convex functionals. Another important remark is that at optimality, the sets $\mathcal{D}_\mathbf{z}^f$ and $\mathcal{D}_\mathbf{z}$ are equal, independently of the sampling from $p_Z$ and of the choice of $N$. Therefore, MMD with Coulomb kernel forces $q_Z$ to be equal to $p_Z$.

It is important to mention that Coulomb kernels represent a generalization of Coulomb's law to any $h$-dimensional Euclidean space.\footnote{In order to see this, consider that for $h=3$ the kernel function in~(\ref{eq:coulomb}) obeys exactly to the Coulomb's law.} 
Therefore, samples from $p_{\mathbf{Z}}$ and $q_{\mathbf{Z}}$ can be regarded as positive and negative 
charged particles, respectively, while the Coulomb kernels induce some global attraction and repulsion forces between them. As a consequence, 
the minimization of the regularizer in~(\ref{eq:empobjective}), with respect to the location of the negative charged particles, allows to find a configuration where the two sets of particles balance each other.
Based on this interpretation, we can highlight the differences between Coulomb and other kernels from previous work~\cite{tolstikhin2018wasserstein} using two simple mono-dimensional cases ($h=1$). The first example consists of three positive particles, located at $-4, 0$ and $4$, and a single negative particle, that is allowed to move freely. In this case, $p_Z(z)=\delta(z+4)+\delta(z)+\delta(z-4)$ and $q_Z(z)=\delta(z-z_1)$, where $z_1$ represents the variable location of the
negative particle. Figure~\ref{fig:demo}(a) and Figure~\ref{fig:demo}(b)
represent the plots of the regularizer in~(\ref{eq:objective}) evaluated at different $z_1$ for the Gaussian, the inverse multiquadratic and the Coulomb kernels, respectively.
The Gaussian and the inverse multiquadratic kernels introduce new local optima and the negative particle is attracted locally to one of the positive charges without being affected by the remaining ones. On the contrary, the Coulomb kernel has only a single minimum. This minimal configuration is the best one, if one considers that all  positive particles exert an attraction force on the negative one. 
As a result the Coulomb kernel induces \textbf{global attraction forces}. 
The second example consists of the same three positive particles and a pair of free negative charges. In this case, $q_Z(z)=\delta(z-z_1)+\delta(z-z_2)$, where $z_1, z_2$ are the locations of the two negative particles. 
Figure~\ref{fig:demo}(d) and Figure~\ref{fig:demo}(e) represent the plots of the regularizer in~(\ref{eq:objective}) evaluated at different $z_1,z_2$ for the Gaussian, the inverse multiquadratic and the Coulomb kernels, respectively. Following the same reasoning of the previous example, we conclude that the Coulomb kernel induces \textbf{global repulsion forces}.\footnote{In this case, there are a pair of minima, corresponding to the permutation of a single configuration.}

It is worth mentioning that these theoretical results are valid
when the optimization is performed on the function space, namely when
minimizing with respect to $f$ and $g$. In reality, the training is
performed on the parameter space of neural networks, which may introduce
local optima due to their non-convex nature. Solving the problem of local minima in the parameter space of neural networks is a very general problem common to deep learning approaches, which is out of the scope of this work. Our aim is to provide a principled objective function with better convergence properties with respect to existing works. 

\subsection{Generalization bound}
The following theorem provides a probabilistic bound on the estimation error between $\hat{\mathcal{L}}(f,g)$ 
and $\mathcal{L}(f,g)$ in~(\ref{eq:empobjective}).
\begin{thm}
    \label{th:bound}
    Given the objective in~(\ref{eq:empobjective}), $h>2$,
    $\Omega_{\mathbf{z}}$ a compact set,
    $\Omega_{\mathbf{x}}=[-M,M]^d$ for positive scalar $M$, and
    a symmetric, continuous and 
    positive definite kernel 
    $k:\Omega_{\mathbf{z}}\times\Omega_{\mathbf{z}}\rightarrow\mathbb{R}$, where
    $0\leq k(\mathbf{z},\mathbf{z}')\leq K$ for all $\mathbf{z},\mathbf{z}'\in\Omega_{\mathbf{z}}$
    with $K=k(\mathbf{z},\mathbf{z})$.
    If the reconstruction error $\|\mathbf{x}{-}g(f(\mathbf{x}))\|^2$ can be made small $\forall \mathbf{x}\in\Omega_{\mathbf{x}}$, such that it can be bounded by a small value $\xi$.
    
    Then, for any $ s,u,v,t>0$
    \begin{align}
        &\text{Pr}\bigg\{|\hat{\mathcal{L}}-\mathcal{L}|{>}t+\lambda (s+u+v)\bigg\}
        \leq 2\exp\bigg\{{-}\frac{2Nt^2}{\xi^2}\bigg\} \nonumber\\
        &\quad+ 2\exp\bigg\{{-}\frac{2\lfloor N/2\rfloor s^2}{K^2}\bigg\}
        + 2\exp\bigg\{{-}\frac{2\lfloor N/2\rfloor u^2}{K^2}\bigg\}\nonumber\\
        &\quad+ 2\exp\bigg\{{-}\frac{2N v^2}{K^2}\bigg\}\nonumber
    \end{align}
\end{thm}

\begin{proof}
    In order to prove the theorem, we first derive the statistical bounds for the reconstruction and the MMD terms separately, and then
    combine them to obtain the final bound.

    Consider the reconstruction error term and define $\xi_{\mathbf{x}}\doteq\|\mathbf{x}{-}g(f(\mathbf{x}))\|^2$.
    Note that $\Omega_{\mathbf{x}}=[-M,M]^d$ and therefore $\xi_{\mathbf{x}}$ is bounded in the interval $[0,4M^2d]$.
    By considering $\xi_{\mathbf{x}}$ a random variable, we can apply the Hoeffding's inequality (see Theorem 2 in~\cite{hoeffding1963prob})
    to obtain the following statistical bound:
    \begin{align}
        \text{P}_0\doteq&\text{Pr}\bigg\{\bigg|\frac{1}{N}{\sum_{\mathbf{x}{\in}\mathcal{D}_\mathbf{x}}}\xi_{\mathbf{x}}
        {-}{\int_{\Omega_{\mathbf{x}}}}\xi_{\mathbf{x}}p_{\mathbf{X}}(\mathbf{x})d\mathbf{x}\bigg|\geq 2\exp\bigg\{{-}\frac{2Nt^2}{\xi^2}\bigg\}
        \label{eq:reconbound}
    \end{align}
    where $t$ is an arbitrary small positive constant.

    We can then proceed to find the bound for the other terms in~(\ref{eq:empobjective}).
    In particular, using the one-sample and two sample U statistics in~\cite{hoeffding1963prob} 
    (see pag. 25), we obtain the following bounds:
    \begin{align}
        \text{P}_1\doteq&\text{Pr}\bigg\{\bigg|\frac{1}{N(N{-}1)}\sum_{\substack{\mathbf{z}_i,\mathbf{z}_j{\in}\mathcal{D}_\mathbf{z}\\j\neq i}}
        k_{i,j}{-}D(p_{\mathbf{Z}}, p_{\mathbf{Z}})\bigg| \nonumber\\
        &\geq s\bigg\}\leq 2\exp\bigg\{\frac{{-}2\lfloor N/2\rfloor s^2}{K^2}\bigg\}\nonumber\\            
        \text{P}_2\doteq&\text{Pr}\bigg\{\bigg|\frac{1}{N(N{-}1)}\sum_{\substack{\mathbf{z}_i,\mathbf{z}_j{\in}\mathcal{D}_\mathbf{z}^f\\j\neq i}}
        k_{i,j}{-}D(q_{\mathbf{Z}}, q_{\mathbf{Z}})\bigg| \nonumber\\
        &\geq u\bigg\}\leq  2\exp\bigg\{\frac{{-}2\lfloor N/2\rfloor u^2}{K^2}\bigg\}\nonumber\\            
        \text{P}_3\doteq&\text{Pr}\bigg\{\bigg|-\frac{2}{N^2}\sum_{\mathbf{z}_i{\in}\mathcal{D}_\mathbf{z}}\sum_{\mathbf{z}_j{\in}\mathcal{D}_\mathbf{z}^f}
        k_{i,j} +2D(p_{\mathbf{Z}}, q_{\mathbf{Z}})\bigg| \nonumber\\
        &\geq v\bigg\}\leq  2\exp\bigg\{\frac{{-}2N v^2}{K^2}\bigg\}
        \label{eq:prob}
    \end{align}
    where $D(p_{\mathbf{Z}}, q_{\mathbf{Z}})\doteq\int_{\Omega_{\mathbf{z}}}\int_{\Omega_{\mathbf{z}}}p_{\mathbf{Z}}(\mathbf{z})q_{\mathbf{Z}}(\mathbf{z}')k(\mathbf{z},\mathbf{z}')d\mathbf{z}d\mathbf{z}'$.
    Then, we can get the following lower bound:
    \begin{align*}
        \sum_{i=0}^3\text{P}_i
        &\geq \text{Pr}\bigg\{\bigg|\frac{1}{N}{\sum_{\mathbf{x}{\in}\mathcal{D}_\mathbf{x}}}\xi_{\mathbf{x}}
        {-}{\int_{\Omega_{\mathbf{x}}}}\xi_{\mathbf{x}}p_{\mathbf{X}}(\mathbf{x})d\mathbf{x}\bigg|\geq t \quad\cup \nonumber\\
        &\quad\quad \bigg|\frac{1}{N(N{-}1)}\sum_{\substack{\mathbf{z}_i,\mathbf{z}_j{\in}\mathcal{D}_\mathbf{z}\\j\neq i}}
        k_{i,j}-D(p_{\mathbf{Z}}, p_{\mathbf{Z}})\bigg|\geq s \quad\cup \nonumber\\
        &\quad\quad \bigg|\frac{1}{N(N{-}1)}\sum_{\substack{\mathbf{z}_i,\mathbf{z}_j{\in}\mathcal{D}_\mathbf{z}^f\\j\neq i}}
        k_{i,j}-D(q_{\mathbf{Z}}, q_{\mathbf{Z}})\bigg|\geq u \quad\cup \nonumber\\
        &\quad\quad\bigg|-\frac{2}{N^2}\sum_{\mathbf{z}_i{\in}\mathcal{D}_\mathbf{z}}\sum_{\mathbf{z}_j{\in}\mathcal{D}_\mathbf{z}^f}
        k_{i,j} +2D(p_{\mathbf{Z}}, q_{\mathbf{Z}})\bigg|\geq v\bigg\} \nonumber\\
        &= \text{Pr}\bigg\{\bigg|\frac{1}{N}{\sum_{\mathbf{x}{\in}\mathcal{D}_\mathbf{x}}}\xi_{\mathbf{x}}
        {-}{\int_{\Omega_{\mathbf{x}}}}\xi_{\mathbf{x}}p_{\mathbf{X}}(\mathbf{x})d\mathbf{x}\bigg|\geq t \quad\cup \nonumber\\
        &\quad\quad \lambda\bigg|\frac{1}{N(N{-}1)}\sum_{\substack{\mathbf{z}_i,\mathbf{z}_j{\in}\mathcal{D}_\mathbf{z}\\j\neq i}}
        k_{i,j}-D(p_{\mathbf{Z}}, p_{\mathbf{Z}})\bigg|\geq \lambda s \quad\cup \nonumber\\
        &\quad\quad \lambda\bigg|\frac{1}{N(N{-}1)}\sum_{\substack{\mathbf{z}_i,\mathbf{z}_j{\in}\mathcal{D}_\mathbf{z}^f\\j\neq i}}
        k_{i,j}-D(q_{\mathbf{Z}}, q_{\mathbf{Z}})\bigg|\geq \lambda u \quad\cup \nonumber\\
        &\quad\quad \lambda\bigg|-\frac{2}{N^2}\sum_{\mathbf{z}_i{\in}\mathcal{D}_\mathbf{z}}\sum_{\mathbf{z}_j{\in}\mathcal{D}_\mathbf{z}^f}
        k_{i,j} +2D(p_{\mathbf{Z}}, q_{\mathbf{Z}})\bigg|\geq \lambda v\bigg\} \nonumber\\       
        &\geq \text{Pr}\bigg\{\bigg|\frac{1}{N}{\sum_{\mathbf{x}{\in}\mathcal{D}_\mathbf{x}}}\xi_{\mathbf{x}}
        {-}{\int_{\Omega_{\mathbf{x}}}}\xi_{\mathbf{x}}p_{\mathbf{X}}(\mathbf{x})d\mathbf{x} + \lambda\frac{1}{N(N{-}1)}\sum_{\substack{\mathbf{z}_i,\mathbf{z}_j{\in}\mathcal{D}_\mathbf{z}\\j\neq i}}
        k_{i,j} \nonumber\\
        &\quad\quad-\lambda D(p_{\mathbf{Z}}, p_{\mathbf{Z}}) +  \frac{1}{N(N{-}1)}\lambda\sum_{\substack{\mathbf{z}_i,\mathbf{z}_j{\in}\mathcal{D}_\mathbf{z}^f\\j\neq i}}
        k_{i,j} \nonumber\\
        &\quad\quad-\lambda D(q_{\mathbf{Z}}, q_{\mathbf{Z}}) -\frac{2}{N^2}\lambda\sum_{\mathbf{z}_i{\in}\mathcal{D}_\mathbf{z}}\sum_{\mathbf{z}_j{\in}\mathcal{D}_\mathbf{z}^f}
        k_{i,j} \nonumber\\
        &\quad\quad+2\lambda D(p_{\mathbf{Z}}, q_{\mathbf{Z}})\bigg|\geq t+\lambda(s+u+v)\bigg\} \nonumber\\        
        &\geq \text{Pr}\bigg\{\bigg|\frac{1}{N}{\sum_{\mathbf{x}{\in}\mathcal{D}_\mathbf{x}}}\xi_{\mathbf{x}}
        +\lambda\bigg[
            \frac{1}{N(N{-}1)}\sum_{\substack{\mathbf{z}_i,\mathbf{z}_j{\in}\mathcal{D}_\mathbf{z}\\j\neq i}}
            k_{i,j}\nonumber\\
            &\quad\quad \frac{1}{N(N{-}1)}\sum_{\substack{\mathbf{z}_i,\mathbf{z}_j{\in}\mathcal{D}_\mathbf{z}^f\\j\neq i}}
            k_{i,j} 
            -\frac{2}{N^2}\sum_{\mathbf{z}_i{\in}\mathcal{D}_\mathbf{z}}\sum_{\mathbf{z}_j{\in}\mathcal{D}_\mathbf{z}^f}
            k_{i,j}
        \bigg]\nonumber\\
        &\quad\quad {-}{\int_{\Omega_{\mathbf{x}}}}\xi_{\mathbf{x}}p_{\mathbf{X}}(\mathbf{x})d\mathbf{x} -\lambda\bigg[
            D(p_{\mathbf{Z}}, p_{\mathbf{Z}})+ D(q_{\mathbf{Z}}, q_{\mathbf{Z}}) + \nonumber\\
            &\quad\quad-2D(p_{\mathbf{Z}}, q_{\mathbf{Z}})
            \bigg]
        \bigg| \geq t+\lambda(s+u+v)\bigg\} \nonumber\\
        &=\text{Pr}\bigg\{\bigg|\hat{\mathcal{L}}
        {-}{\int_{\Omega_{\mathbf{x}}}}\xi_{\mathbf{x}}p_{\mathbf{X}}(\mathbf{x})d\mathbf{x}\nonumber\\
            &\quad\quad
            -\lambda\bigg[
            \int_{\Omega_{\mathbf{z}}}\int_{\Omega_{\mathbf{z}}}
            (p_{\mathbf{Z}}(\mathbf{z})-q_{\mathbf{Z}}(\mathbf{z}))
            (p_{\mathbf{Z}}(\mathbf{z}')\nonumber\\
            &\quad\quad
            -q_{\mathbf{Z}}(\mathbf{z}'))
            k(\mathbf{z},\mathbf{z}')d\mathbf{z}d\mathbf{z}'
            \bigg]\bigg| \geq t+\lambda(s+u+v)\bigg\} \nonumber\\
        &=\text{Pr}\bigg\{\bigg|\hat{\mathcal{L}}
        {-}\mathcal{L}\bigg| \geq t+\lambda(s+u+v)\bigg\}
        \label{eq:finite}
    \end{align*}
    where the first inequality is obtained by applying the union bound. Finally, by exploiting also the results in~(\ref{eq:reconbound}),~(\ref{eq:prob})
    we get the desired bound.
\end{proof}
Theorem~\ref{th:bound} provides a probabilistic bound on the estimation error between $\hat{\mathcal{L}}(f,g)$ and $\mathcal{L}(f,g)$. 
The bound consists of four terms which vanish when $N$ is large. It is important to mention that, while the last three
terms can be made arbitrarily small, by choosing appropriate values for $s,u,v$ and $\lambda$, the first term depends mainly on 
on the value of $\xi$, which can be controlled by modifying the capacity of the encoding and the decoding networks. Therefore, \textbf{we can improve the generalization performance of the model by controlling the capacity of the networks as long as $\xi$ can be made small}. This is confirmed also in practice, as shown in the experimental section.

{\section{Related Work}
\label{sec:related}}
The most promising research directions for implicit generative models are generative adversarial networks (GANs) and autoencoder-based models.

GANs~\cite{goodfellow2014generative} cast the problem of density estimation
as a mini-max game between two neural networks, namely a discriminator, that tries to distinguish
between true and generated samples, and a generator, that tries to produce samples similar to the
true ones, to fool the discriminator. GANs are notoriously difficult to train, usually requiring careful design strategies for network architectures in~\cite{radford2015unsupervised}. Some of 
the most known issues are (i) the problem of vanishing gradients in~\cite{arjovsky2017towards}, which happens
when the output of the discriminator is saturated, because true and generated data are perfectly classified, 
and no more gradient information is provided to the generator, (ii) the problem of mode collapse in~\cite{metz2017unrolled}, 
which happens when the samples from the generator collapse to a single point corresponding to the maximum
output value of the discriminator, and (iii) the problem of instability associated with the failure of
convergence, which is due to the intrinsic nature of the mini-max problem.
Different line of works~\cite{goodfellow2014generative}~\cite{salimans2016improved}~\cite{durugkar2017generative}~\cite{mescheder2017numerics}~\cite{metz2017unrolled}~\cite{karras2017progressive}, have proposed effective solutions to overcome the aforementioned issues with GANs. However, all these strategies have either poor theoretical motivation or they are guaranteed to converge only locally.

Another research direction for GANs consists on using integral probability metrics~\cite{muller1997integral} 
as optimization objective. In particular, the maximum mean discrepancy~\cite{gretton08kernel} can be 
used to measure the distance between $p_{\mathbf{X}}$ and $q_{\mathbf{Y}}$ and train the generator network. 
The general problem is formulated in the following way:
\begin{equation}
    \inf_{g\in\mathcal{G}}\sup_{f\in\mathcal{F}}\big\{E_{\mathbf{x}\sim p_{\mathbf{X}}}[f(\mathbf{x})] 
    - E_{\mathbf{y}\sim q_{\mathbf{Y}}}[f(\mathbf{y})]\big\} \nonumber\\
\end{equation}
In generative moment matching networks~\cite{li2015gmmn,dziugaite2015mmd} $\mathcal{F}$ is a RKHS,
which is induced by the Gaussian kernel. Note that a major limitation of these models is the curse of dimensionality, since the similarity scores associated with the kernel function are directly computed in the sample space, as explained in~\cite{ramdas2015decreasing}. The work of~\cite{li2017mmd} introduces an
encoding function to represent data in a more compact way and distances are computed in the latent
representation, thus solving the problem of dimensionality.
\cite{mroueh2017mcgan} propose to extend the maximum mean discrepancy and include also covariance
statistics to ensure better stability. The work of~\cite{tolstikhin2018wasserstein} generalizes the
computation of the distance between the encoded distribution and the prior to other divergences, thus
proposing two different solutions: the first one consists of using the Jensen-Shannon divergence, showing also the equivalence to adversarial autoencoders, and the second one consists of using the maximum-mean discrepancy. The choice of the kernel function in this second case is of fundamental importance to ensure the global convergence of gradient-descent algorithms. As we have already shown in previous section, suboptimal choices of the kernel function, like the ones used by the authors, introduce local optima
in the function space and therefore do not have the same convergence property of our model. The work by \cite{unterthiner2017coulomb} use Coulomb kernels under the GANs' framework. Nevertheless, the computation of distances is performed directly in the sample space, thus being negatively affected by the curse of dimensionality.

There exists other autoencoder-based models that are inspired by the adversarial game of GANs. 
\cite{chen2016infogan} add an autoencoder network to the original GANs for reconstructing 
part of the latent code. The identical works of~\cite{donahue2017adversarial} and~\cite{dumoulin2017adversarially} 
propose to add an encoding function together with the generator and perform an adversarial game to ensure
that the joint density on the input/output of the generator agrees with the joint density of the input/output
of the encoder. They prove that the optimal solution is achieved when the generator and the encoder
are invertible. In practice, they fail to guarantee the convergence to that solution due to the adversarial
nature of the game. \cite{srivastava2017veegan} extend the previous works by explicitly imposing
the invertibility condition. They achieve this by adding a term to the generator objective that
computes the reconstruction error on the latent space. Adversarial autoencoders by~\cite{makhzani2013adversarial}
are similar to these approaches with the only differences that the estimation of the reconstruction error
is performed in the sample space, while the adversarial game is performed only in the latent space.
It is important to mention that all of these works are based on a mini-max problem, while our method 
solves a simple minimization problem, which behaves better in terms of training convergence.

Variational autoencoders (VAEs) by~\cite{kingma2014vae,rezende2014stochastic} represent another family of autoencoder-based models. The framework is based on minimizing the Kullback-Leibler (KL) divergence between the approximate posterior distribution defined by the encoder and the true prior $p_{\mathbf{Z}}$ 
(which consists of a surrogate for the negative log-likelihood of training data). Practically speaking, 
the stochastic encoder used in variational autoencoders is driven to produce latent representations that 
can be similar among different input samples, thus generating conflicts during reconstruction. 
A deterministic encoder could ideally solve this problem, but unfortunately the KL divergence is not defined 
for such case. There are several variations for VAEs. For example, the work in~\cite{mescheder2017adversarial}
proposes to use the adversarial game of GANs to learn better approximate posterior distributions in VAEs.
Nevertheless, the method is still based on a mini-max problem. Recently, \cite{dai2019diagnosing} propose a training strategy based on a cascade of two VAEs to deal with the limitations implied by the KL divergence. In particular, the authors train a first VAE on the training data and then train a second VAE on the learnt latent representations. This second step is fundamental to improve the matching between the posterior density and the prior with respect to what is done in the first stage. Therefore, this solution implicitly considers the mitigation of local minima from the level of architecture design. However, the 2-stage procedure does not prevent local minima induced by the combination of the two addends in the two objectives.
To the best of our knowledge, only \cite{lucas2019understanding} are aware of the problem of local minima in generative autoencoders. The authors analyze theoretically the behaviour of simple linear VAEs and show that the phenomenon known as posterior collapse\footnote{i.e. the posterior over some latent variables matches the prior, with the consequence that those latent variables ignore encoder inputs.} is due to the problem of local minima (or equivalently local maxima, when considering to maximize the ELBO).

{\section{Experiments}
\label{sec:experiments}}
We evaluate the performance of our model (CouAE) against the baseline of Variational Autoencoders (VAE)~\cite{kingma2014vae,rezende2014stochastic} and  Wasserstein Autoencoders (WAE)~\cite{tolstikhin2018wasserstein}. All experiments are performed on two synthetic datasets, to simulate scenarios with low and high dimensional feature spaces and on a real-world faces' dataset, namely CelebA 64x64.\footnote{We choose $\lambda=100$ for all experiments except the ones on the low-dimensional embedding dataset, in which we use $\lambda=1$ to avoid numerical instabilities.}

We distinguish between two sets of experiments. The first set confirms the usefulness of using the MMD coupled with the Coulomb kernel, while the second one 
aims at validating the generalization error bound in Theorem~\ref{th:bound}.

%
\subsection{Comparison with other autoencoders}
\begin{table}
    \caption{Comparison among different autoencoders on different datasets.}
    \label{tab:comparison}
    \centering
    \begin{tabular}{lcccc}
    \hline
    \textbf{Eval. Metric} & \textbf{Data/Method} & VAE & WAE & CouAE \\
    \hline
    Test Log-likel. & Grid & -4.4$\pm$0.2 & -6.4$\pm$1.1 & \textbf{-4.3$\pm$0.1}
    \\
    FID & CelebA & 63 & 55 & \textbf{47}
    \\
    \hline
    \end{tabular}
\end{table}
\begin{figure}
    \centering
        \subfigure[True Data]{\includegraphics[width=0.24\linewidth]{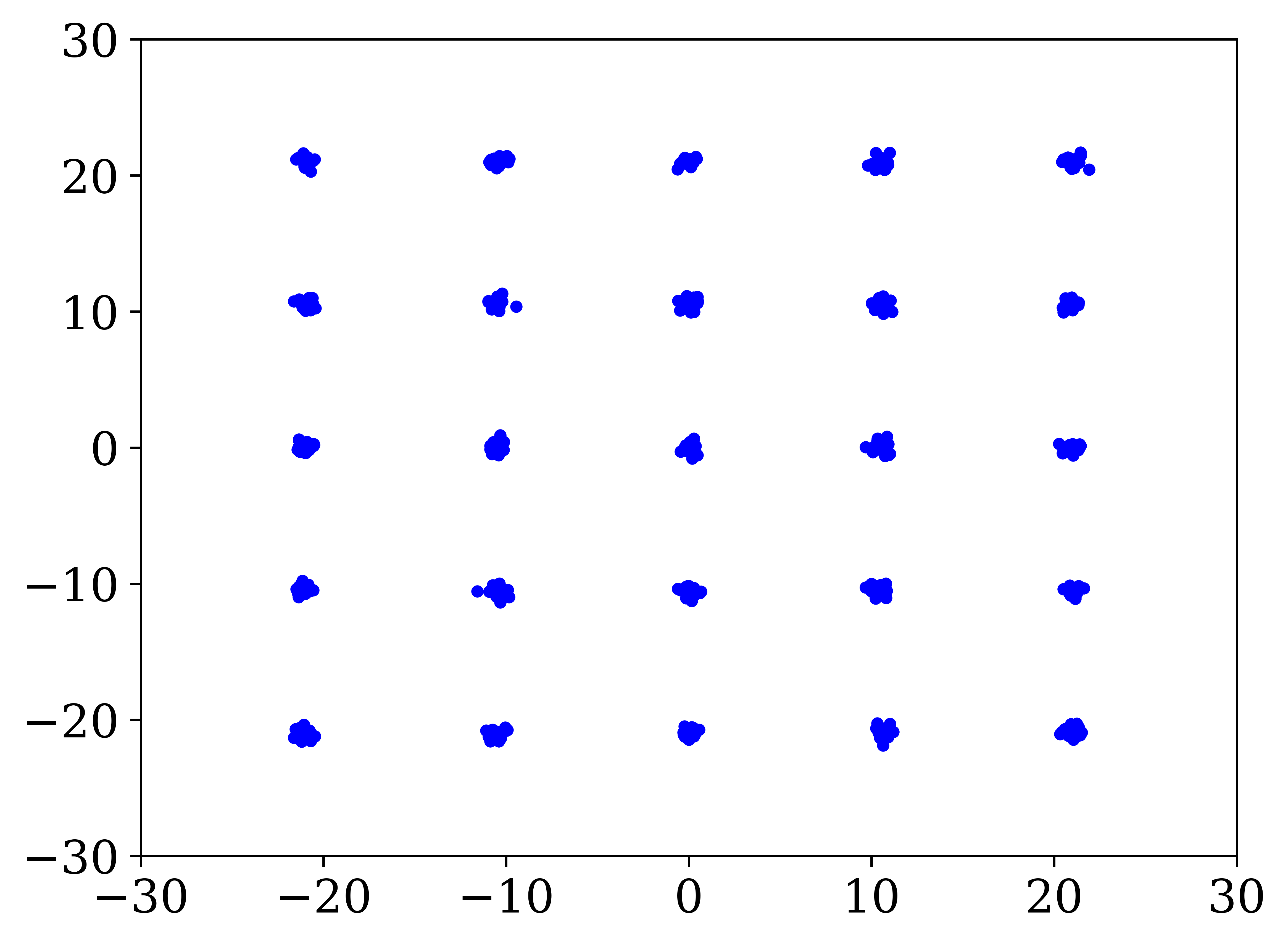}}
        \subfigure[VAE]{\includegraphics[width=0.24\linewidth]{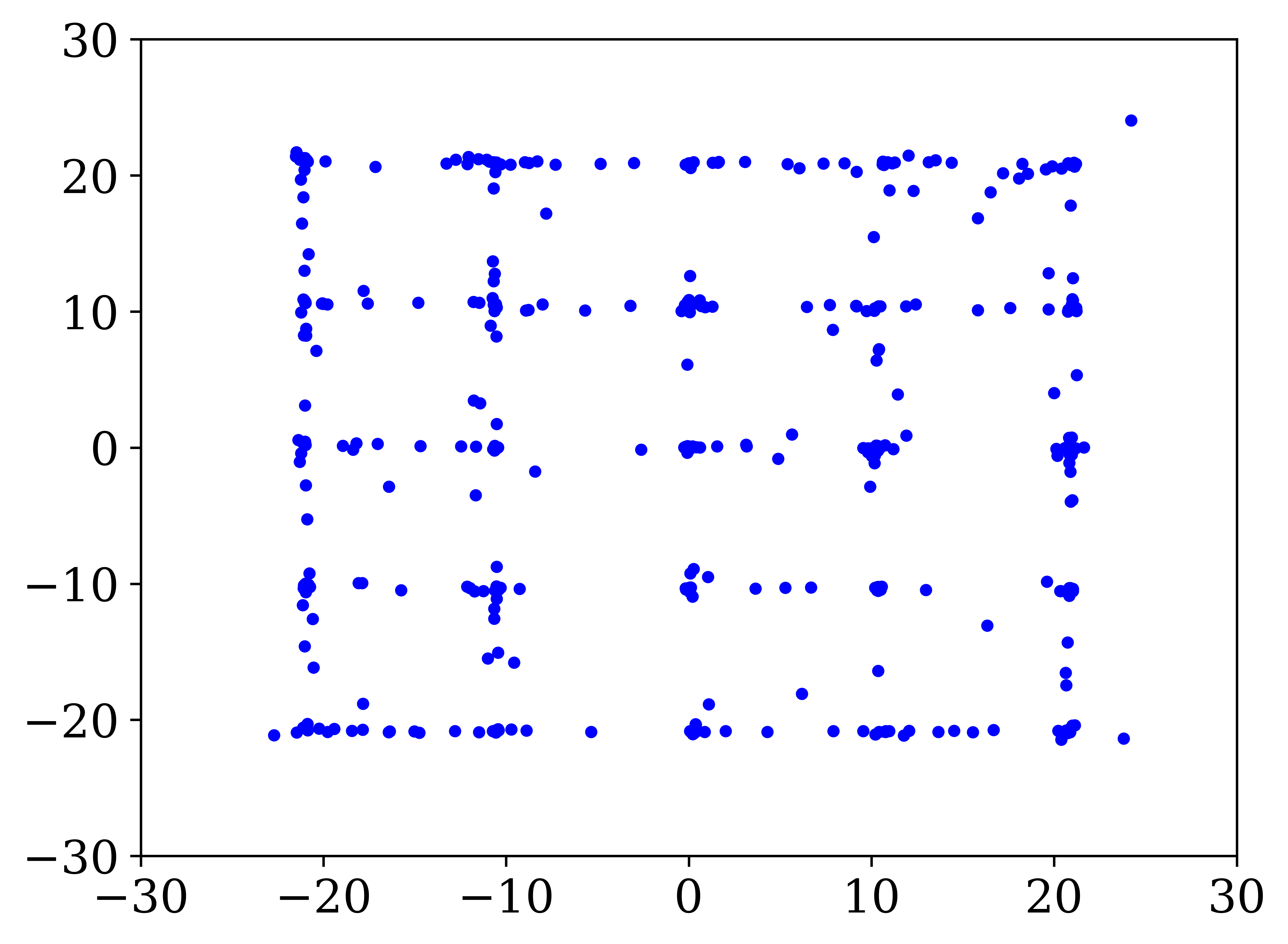}}
        \subfigure[WAE]{\includegraphics[width=0.24\linewidth]{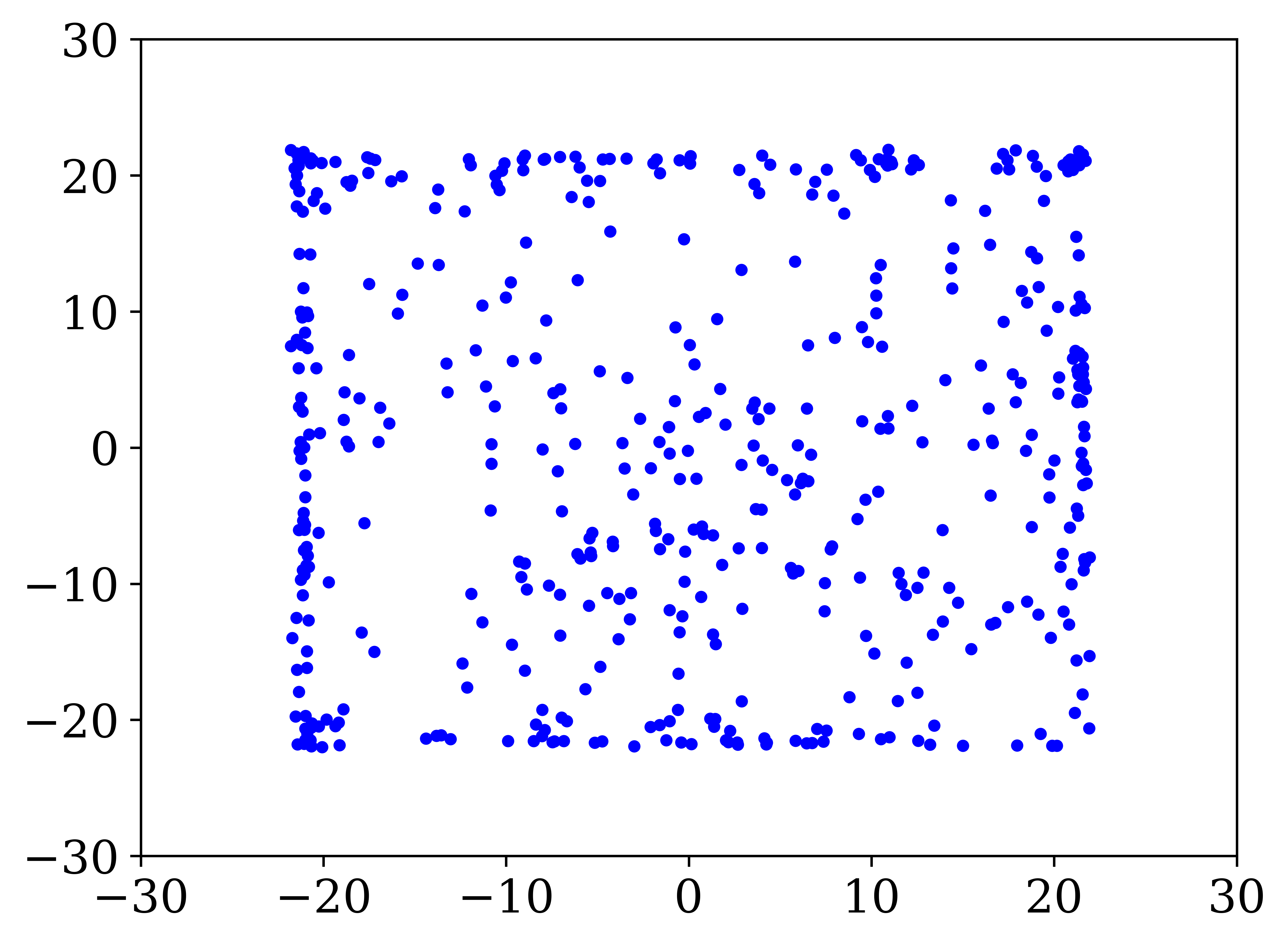}}
        \subfigure[CouAE]{\includegraphics[width=0.24\linewidth]{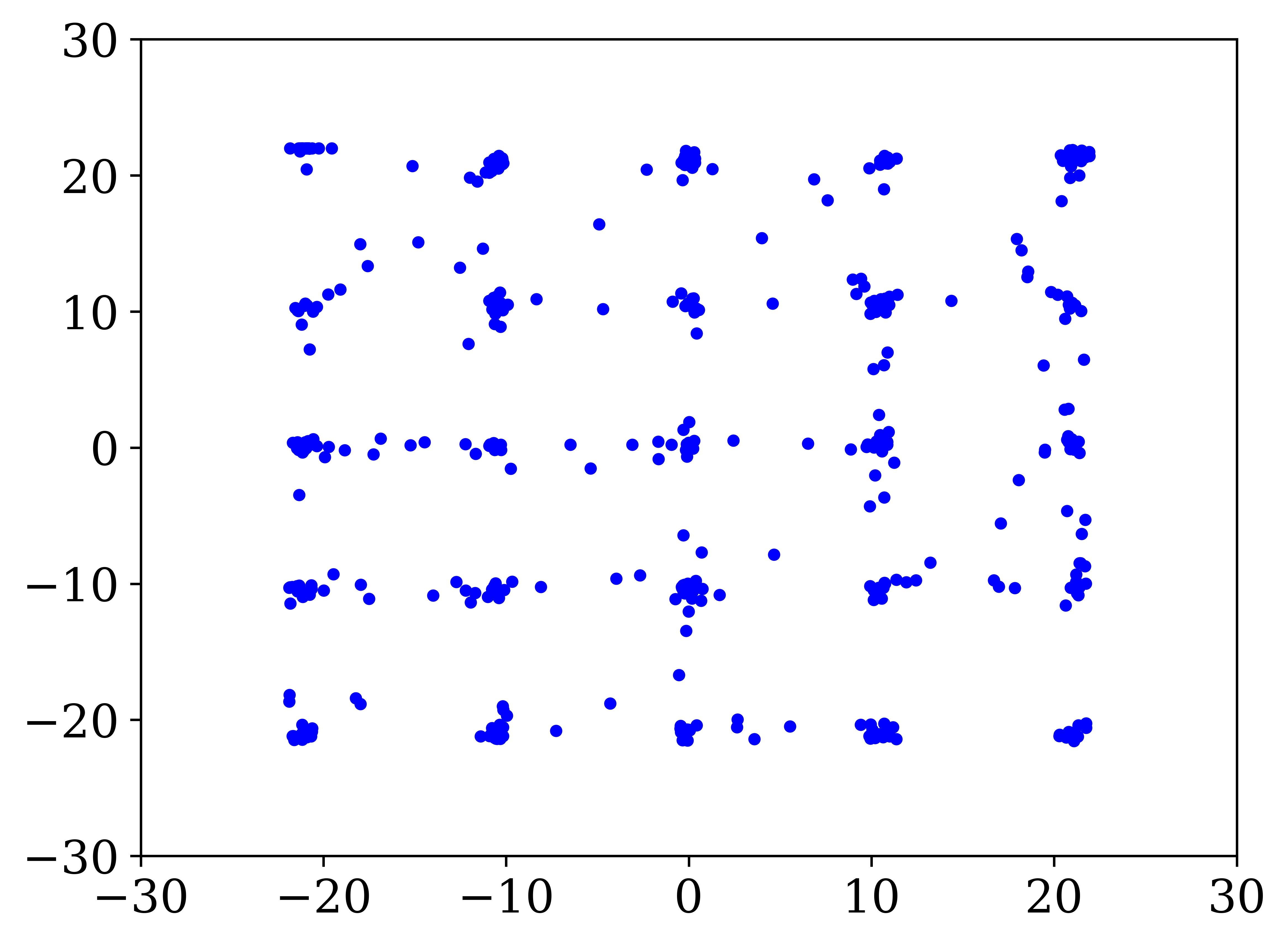}}
    \caption{Generated data from different models on grid dataset.}
    \label{fig:25modes}
\end{figure}
\begin{figure}
    \centering
        \subfigure[VAE]{\includegraphics[trim=130 130 130 130,clip,width=0.31\linewidth]{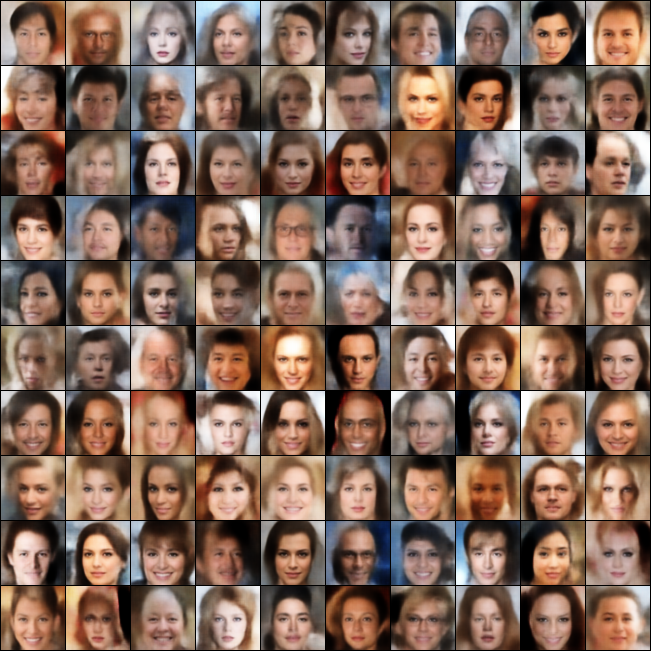}}
        \subfigure[WAE]{\includegraphics[trim=130 130 130 130,clip,width=0.31\linewidth]{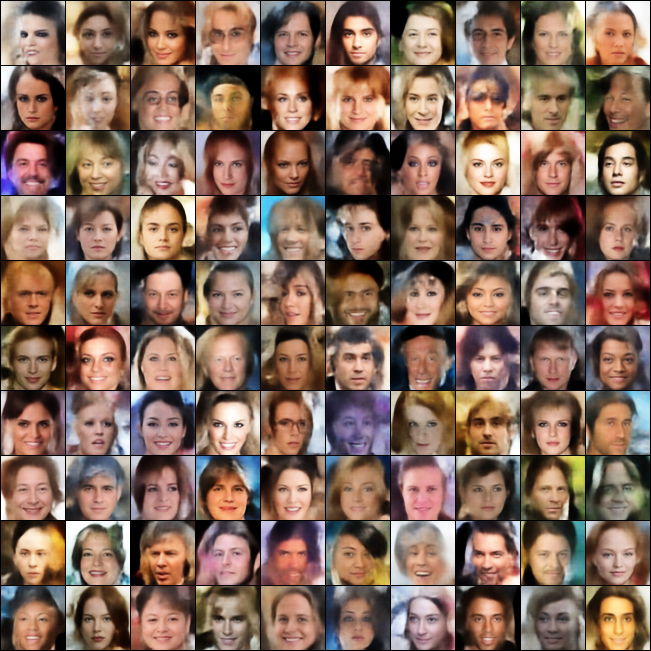}}
        \subfigure[CouAE]{\includegraphics[trim=12 12 12 12,clip,width=0.31\linewidth]{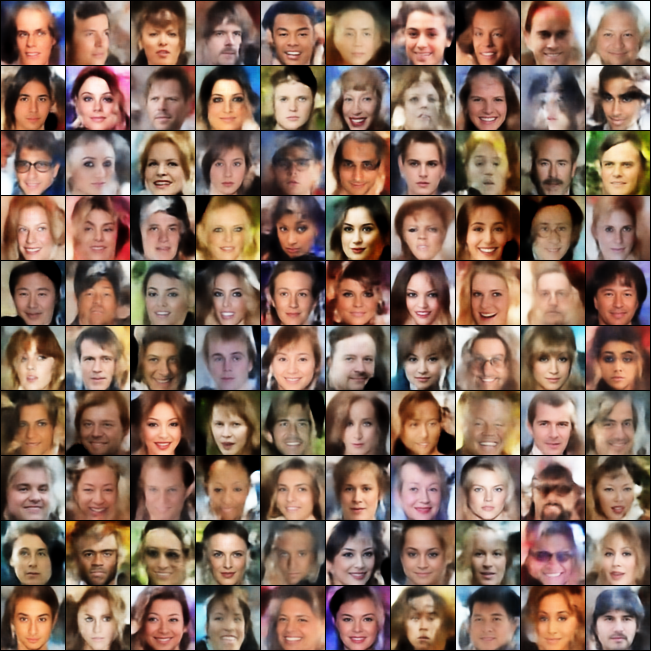}}
        \caption{Generated samples on CelebA.}
    \label{fig:celeba}
\end{figure}
We start by comparing the approaches on a two-dimensional dataset consisting of 25 isotropic Gaussians placed according to a grid (see Figure~\ref{fig:25modes}(a)), hereafter called the grid dataset~\cite{lim2017geometric}. 
The training dataset contains $500$ samples generated from the true density.

Following the methodology of other works (see for example~\cite{lim2017geometric,unterthiner2017coulomb}, we choose fully connected Multilayer Perceptrons with two hidden layers (128 neurons each) at both the encoder and the decoder and set $h=2$. All models are trained for $3.10^6$ iterations using Adam optimizer with learning rate $10^{-3}$. Models are evaluated qualitatively by visually inspecting generated samples and quantitatively by computing the log-likelihood on test data. To compute the log-likelihood, we first apply kernel density estimation using a Gaussian kernel on $10^4$ generated samples\footnote{Bandwidth is selected from a set of $10$ values logarithmically spaced in $[10^{-3},10^{1.5}]$.}
and then evaluate the log-likelihood on $10^4$ test samples from the true distribution. Results are averaged over $10$ repetitions.

Figure~\ref{fig:25modes} shows samples generated by all models, while Table~\ref{tab:comparison} provides
quantitative results in terms of test log-likelihood. These experiments highlight the fact that an improper choice of the kernel function may lead to worse performance. In fact, note that WAE does not perform as good as our proposed solution.


For the experiments on CelebA, we follow the settings used in~\cite{tolstikhin2018wasserstein}. In particular, we choose a DCGAN architecture~\cite{radford2015unsupervised} and train all models for $10^5$ iterations with a learning rate of $0.0005$.\footnote{Similarly to the low-dimensional embedding dataset, we choose $\beta=2$.} For the competitors, we run the simulations using the implementation of~\cite{radford2015unsupervised}. Figure~\ref{fig:celeba} shows samples generated by all models, while Table~\ref{tab:comparison} provides
quantitative results in terms of test FID~\cite{heusel2017gans}. These experiments confirm the findings observed on the grid and the low-dimensional embedding datasets, namely that the choice of using the MMD coupled with the Coulomb kernel provides significant improvements over VAEs and WAEs.

\subsection{Validation of generalization bound}
\begin{table}
    \caption{Experimental validation of generalization bound for CouAE.
    A scaling factor is applied to the number of neurons in each hidden layer to control the capacity of the encoder and the decoder.}
    \label{tab:capacity}
    \centering
    \resizebox{\linewidth}{!}{
    \begin{tabular}{lcccc}
    \hline
    \textbf{Eval. Metric} & \textbf{Data/Width factor} & $\times0.25$ & $\times0.5$ & $\times1$ \\
    \hline
    Test Log-likel. & Grid & -5.8$\pm$0.4 & -4.8$\pm$0.4 & -4.3$\pm$0.1 \\
    \hline
    \textbf{Eval. Metric} & \textbf{Data/Width factor} & $\times0.25$ & $\times0.5$ & $\times1$ \\
    \hline
    FID & CelebA & 53 & 51 & 47
    \\
    \hline
    \end{tabular}}
\end{table}
\begin{figure}
    \centering
        \subfigure[$\times0.25$]{\includegraphics[trim=35 0 6 0,clip,width=0.4\linewidth]{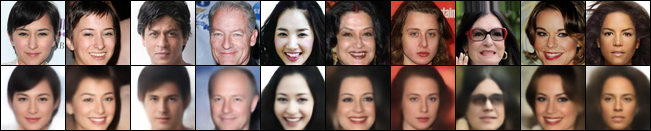}}
        \subfigure[$\times0.5$]{\includegraphics[trim=35 0 6 0,clip,width=0.4\linewidth]{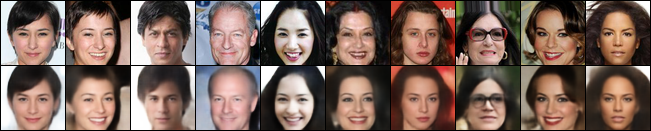}}\\
        \subfigure[$\times1$]{\includegraphics[trim=35 0 6 0,clip,width=0.4\linewidth]{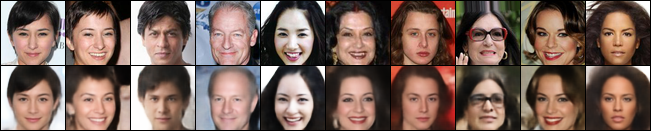}}
        \caption{Reconstruction of test images for different width factors. The top and the bottom row of each case contains the original and reconstructed test images.}
    \label{fig:recon}
\end{figure}
To validate the properties of the generalization bound in Theorem~\ref{th:bound}, we perform experiments on the same datasets used in the previous set of experiments and analyze the performance of CouAE as the capacity of the encoder and the decoder networks changes. In particular, we use the total number of hidden neurons as a proxy to measure the capacity of the model and vary this number according to different scaling factors. Table~\ref{tab:capacity} provides a quantitative analysis of the generalization performance of CouAE. In all cases, we see that an increase of capacity translates into an improvement of the performance. Nevertheless, it is important to mention that there is a limit on the growth of the networks' capacity. As suggested by Theorem~\ref{th:bound}, the growth is mainly limited by $\xi$, which could be estimated averaging the reconstruction error on the train and the validation data. In fact, there is no additional benefit to consider larger networks, once $\xi$ is at its minimum, as the bound in Theorem~\ref{th:bound} is dominated by the MMD term, which can be improved only by increasing the number of samples. Figure~\ref{fig:recon} provides a more qualitative analysis on the relation between generalization and reconstruction error. In particular, we visualize some reconstructed test images from the model and see that an increase of the network capacity allows to capture more details about the original images.

It is important to mention that there are also other architectural factors, which may affect the estimation of $\xi$, and which is worth considering in future research. Some examples are the depth of the networks and the use of residual connections (to mitigate the problem of local minima~\cite{li2018visualizing}).

{\section{Conclusions}
\label{sec:conclusions}}
In this work, we have proposed new theoretical insights on MMD-based autoencoders. In particular, (i) we have proved that MMD coupled with Coulomb kernels has convergence properties similar to convex functionals and shown that these properties have also an impact on the performance of autoencoders, and (ii) we have provided a probabilistic bound on the generalization performance and given principled insights on it.

\bibliographystyle{ecai}
\bibliography{ecai}

\appendix
%
%
\section{Properties of regularizer in (1)}

    \begin{lemma}
        Given $k:\Omega_{\mathbf{z}}\times\Omega_{\mathbf{z}}\rightarrow\mathbb{R}$ a symmetric positive
        definite kernel, then:
        \begin{itemize}
            \item[(a)]~\cite{aronszajn1950theory} there exists a unique Hilbert space $\mathcal{H}$ of 
            real-valued functions over $\Omega_{\mathbf{z}}$, for which $k$ is
            a reproducing kernel. $\mathcal{H}$ is therefore a Reproducing kernel
            Hilbert Space (RKHS).
            \item[(b)] For all $\ell\in\mathcal{H}$
            $$\int_{\Omega_{\mathbf{z}}}\int_{\Omega_{\mathbf{z}}}
            \phi_t(\mathbf{z})\phi_t(\mathbf{z}')k(\mathbf{z}{,}\mathbf{z}')d\mathbf{z}d\mathbf{z}'=\text{MMD}^2(p_{\mathbf{Z}},q_{\mathbf{Z}})$$
            where
            $$\text{MMD}(p_{\mathbf{Z}},q_{\mathbf{Z}})\doteq\sup_{\|\ell\|_\mathcal{H}\leq 1}\big\{E_{\mathbf{z}\sim p_{\mathbf{Z}}}[\ell(\mathbf{z})] 
            - E_{\mathbf{z}\sim q_{\mathbf{Z}}}[\ell(\mathbf{z})]\big\}$$
            is the maximum mean discrepancy between $p_{\mathbf{Z}}(\mathbf{z})$ and $q_{\mathbf{Z}}(\mathbf{z})$.
            \item[(c)]~\cite{gretton08kernel} Let $\mathcal{H}$ be defined as in (b), then $\text{MMD}(p_{\mathbf{Z}},q_{\mathbf{Z}})=0$ if and only if
            $p_{\mathbf{Z}}(\mathbf{z})=q_{\mathbf{Z}}(\mathbf{z})$.
        \end{itemize}
    \end{lemma}
    \begin{proof}
        (a) follows directly from the Moore-Aronszajn theorem~\cite{aronszajn1950theory}. 
        
        Now we prove statement (b).
        For the sake of notation compactness, define $J\doteq\int_{\Omega_{\mathbf{z}}}\int_{\Omega_{\mathbf{z}}}
        \phi_t(\mathbf{z})\phi_t(\mathbf{z}')k(\mathbf{z}{,}\mathbf{z}')d\mathbf{z}d\mathbf{z}'$. Therefore,
        \begin{align}
            J
            &=
            \int_{\Omega_{\mathbf{z}}}\int_{\Omega_{\mathbf{z}}}
            p_{\mathbf{Z}}(\mathbf{z})p_{\mathbf{Z}}(\mathbf{z}')k(\mathbf{z}{,}\mathbf{z}')d\mathbf{z}d\mathbf{z}' -\nonumber\\
            &\quad - \int_{\Omega_{\mathbf{z}}}\int_{\Omega_{\mathbf{z}}}
            p_{\mathbf{Z}}(\mathbf{z})q(\mathbf{z}')k(\mathbf{z}{,}\mathbf{z}')d\mathbf{z}d\mathbf{z}' -\nonumber\\
            &\quad - \int_{\Omega_{\mathbf{z}}}\int_{\Omega_{\mathbf{z}}}
            p_{\mathbf{Z}}(\mathbf{z}')q_{\mathbf{Z}}(\mathbf{z})k(\mathbf{z}{,}\mathbf{z}')d\mathbf{z}d\mathbf{z}' +\nonumber\\
            &\quad
            + \int_{\Omega_{\mathbf{z}}}\int_{\Omega_{\mathbf{z}}}
            q(\mathbf{z}')q_{\mathbf{Z}}(\mathbf{z})k(\mathbf{z}{,}\mathbf{z}')d\mathbf{z}d\mathbf{z}'\nonumber\\
            &= 
            \int_{\Omega_{\mathbf{z}}}\int_{\Omega_{\mathbf{z}}}
            p_{\mathbf{Z}}(\mathbf{z})p_{\mathbf{Z}}(\mathbf{z}')\langle r(\mathbf{z}),r(\mathbf{z}')\rangle_{\mathcal{H}}d\mathbf{z}d\mathbf{z}' -\nonumber\\
            &\quad
            - \int_{\Omega_{\mathbf{z}}}\int_{\Omega_{\mathbf{z}}}
            p_{\mathbf{Z}}(\mathbf{z})q(\mathbf{z}')\langle r(\mathbf{z}),r(\mathbf{z}')\rangle_{\mathcal{H}}d\mathbf{z}d\mathbf{z}' -\nonumber\\
            &\quad - \int_{\Omega_{\mathbf{z}}}\int_{\Omega_{\mathbf{z}}}
            p_{\mathbf{Z}}(\mathbf{z}')q_{\mathbf{Z}}(\mathbf{z})\langle r(\mathbf{z}),r(\mathbf{z}')\rangle_{\mathcal{H}}d\mathbf{z}d\mathbf{z}' +\nonumber\\
            &\quad+
            \int_{\Omega_{\mathbf{z}}}\int_{\Omega_{\mathbf{z}}}
            q(\mathbf{z}')q_{\mathbf{Z}}(\mathbf{z})\langle r(\mathbf{z}),r(\mathbf{z}')\rangle_{\mathcal{H}}d\mathbf{z}d\mathbf{z}'\nonumber\\
            &= 
            \langle E_{\mathbf{z}\sim p_{\mathbf{Z}}}[r(\mathbf{z})],E_{\mathbf{z}'\sim p_{\mathbf{Z}}}[r(\mathbf{z}')])\rangle_{\mathcal{H}} -\nonumber\\
            &\quad-
            \langle E_{\mathbf{z}\sim p_{\mathbf{Z}}}[r(\mathbf{z})],E_{\mathbf{z}'\sim q}[r(\mathbf{z}')])\rangle_{\mathcal{H}} -\nonumber\\
            &\quad
            -\langle E_{\mathbf{z}'\sim p_{\mathbf{Z}}}[r(\mathbf{z}')],E_{\mathbf{z}\sim q}[r(\mathbf{z})])\rangle_{\mathcal{H}} +\nonumber\\
            &\quad+
            \langle E_{\mathbf{z}'\sim q}[r(\mathbf{z}')],E_{\mathbf{z}\sim q}[r(\mathbf{z})])\rangle_{\mathcal{H}}
            \label{eq:temp}
        \end{align}
        Note that the second equality in~(\ref{eq:temp}) follows from the fact that
        $k(\mathbf{z}{,}\mathbf{z}')=\langle r(\mathbf{z}),r(\mathbf{z}')\rangle_{\mathcal{H}}$ 
        for a unique $r\in\mathcal{H}${,}\footnote{This is a classical result due to the Riesz representation theorem.} where $\langle\cdot{,}\cdot\rangle_{\mathcal{H}}$
        is the inner product of $\mathcal{H}$.
        If we define $\mu_{p_{\mathbf{Z}}}\doteq E_{\mathbf{z}\sim p_{\mathbf{Z}}}[r(\mathbf{z})]$ 
        and $\mu_q\doteq E_{\mathbf{z}\sim q}[r(\mathbf{z})]${,}\footnote{Their existence can be guaranteed assuming that $\|\mu_{p_{\mathbf{Z}}}\|_{\mathcal{H}}^2<\infty$
        and $\|\mu_q\|_{\mathcal{H}}^2<\infty$. In other words, $E_{\mathbf{z}{,}\mathbf{z}'\sim p_{\mathbf{Z}}}[k(\mathbf{z}{,}\mathbf{z}')]<\infty$ and 
        $E_{\mathbf{z}{,}\mathbf{z}'\sim q}[k(\mathbf{z}{,}\mathbf{z}')]<\infty$.} then~(\ref{eq:temp}) can be rewritten in the following way:
        \begin{align}
            J &= 
            \langle \mu_{p_{\mathbf{Z}}} {,}\mu_{p_{\mathbf{Z}}}\rangle_{\mathcal{H}} -
            \langle \mu_{p_{\mathbf{Z}}}{,}\mu_q\rangle_{\mathcal{H}}- 
            \langle \mu_{p_{\mathbf{Z}}}{,}\mu_q\rangle_{\mathcal{H}}+
            \langle \mu_q{,}\mu_q\rangle_{\mathcal{H}} \nonumber\\
            &=
            \langle \mu_{p_{\mathbf{Z}}}-\mu_q{,}\mu_{p_{\mathbf{Z}}}-\mu_q\rangle_{\mathcal{H}}\nonumber\\
            &=
            \|\mu_{p_{\mathbf{Z}}}-\mu_q\|_{\mathcal{H}}^2
            \label{eq:temp2}
        \end{align}
        Notice that 
        \begin{align}
            \|\mu_{p_{\mathbf{Z}}}{-}\mu_q\|_{\mathcal{H}}
            &= 
            \Big\langle \mu_{p_{\mathbf{Z}}}-\mu_q{,}\frac{\mu_{p_{\mathbf{Z}}}-\mu_q}{\|\mu_{p_{\mathbf{Z}}}-\mu_q\|_{\mathcal{H}}}\Big\rangle_{\mathcal{H}} \nonumber\\
            &= 
            \sup_{\|\ell\|_{\mathcal{H}}\leq 1}\Big\{
                \langle \mu_{p_{\mathbf{Z}}}-\mu_q,\ell\rangle_{\mathcal{H}}
            \Big\}\nonumber\\
            &=
            \sup_{\|\ell\|_{\mathcal{H}}\leq 1}\Big\{
                E_{\mathbf{z}\sim p_{\mathbf{Z}}}[\langle r(\mathbf{z}), \ell\rangle_{\mathcal{H}}]-
                E_{\mathbf{z}\sim q}[\langle r(\mathbf{z}), \ell\rangle_{\mathcal{H}}]         
            \Big\}\nonumber\\
            &=
            \sup_{\|\ell\|_\mathcal{H}\leq 1}\big\{E_{\mathbf{z}\sim p_{\mathbf{Z}}}[\ell(\mathbf{z})] 
            - E_{\mathbf{z}\sim q}[\ell(\mathbf{z})]\big\} \nonumber\\
            &=\text{MMD}(p_{\mathbf{Z}},q_{\mathbf{Z}})\nonumber
        \end{align}
        Substituting this result into~(\ref{eq:temp2}) concludes the proof of the statement.

        Statement (c) is equivalent to Theorem 3 in~\cite{gretton08kernel}.
    \end{proof}

%
%
\section{Sufficient conditions for global convergence}
This section aims at clarifying the theory proposed in~\cite{hochreiter2005optimal}. Note that the authors have focused on proving that the Poisson's equation is necessary for achieving global convergence. Here, we prove that this equation represents a sufficient condition to guarantee the global convergence on the regularizer in~(1), thus motivating the use of Coulomb kernels. The result of the following proposition implies that the loss function associated with the second addend in~(1) is free from saddle points and all local minima are global.
\begin{proposition}
    Assume that the encoder network has enough capacity to
    achieve the global minimum of the second addend in~(1). Furthermore, assume that the kernel function
    satisfies the Poisson's equation, viz. $\nabla_{\mathbf{z}}^2 k(\mathbf{z},\mathbf{z}')=-\delta(\mathbf{z}-\mathbf{z}')$.
    Then,
    \begin{align}
        &|\phi_t(\mathbf{z}_{max})|=\frac{1}{\sqrt{2t+(\phi_0(\mathbf{z}_{max}))^{-2}}}
    \end{align}
    where $\phi_t(\cdot)$ represents $\phi(\cdot)$ at iteration $t$, 
    while $\mathbf{z}_{max}=\text{arg}\max_{\mathbf{z}}\|\phi_t(\mathbf{z})\|$.
    Therefore, gradient descent-based training converges to the global minimum of the second addend in~(1)
    and at the global minimum $\phi(\mathbf{z})=0$ for all $\mathbf{z}\in\Omega_{\mathbf{z}}$.
\end{proposition}
\begin{proof}
    By defining the potential function at location $\mathbf{z}$, namely $\Psi(\mathbf{z})\doteq\int_{\Omega_{\mathbf{z}}}\phi_t(\mathbf{z}')k(\mathbf{z},\mathbf{z}')d\mathbf{z}'$, the second addend in~(1) can be rewritten in the following way:
    \begin{align}
        \int_{\Omega_{\mathbf{z}}}\int_{\Omega_{\mathbf{z}}}
        \phi(\mathbf{z})\phi(\mathbf{z}')k(\mathbf{z},\mathbf{z}')d\mathbf{z}d\mathbf{z}' &=
        \int_{\Omega_{\mathbf{z}}}\phi_t(\mathbf{z})\Psi(\mathbf{z})d\mathbf{z}\nonumber\\
        &\doteq\int_{\Omega_{\mathbf{z}}}J(\mathbf{z})d\mathbf{z}
        \label{eq:mmd}
    \end{align}
    The overall objective~(\ref{eq:mmd}) is computed by summing the term $J(\mathbf{z})$ for all 
    $\mathbf{z}\in\Omega_{\mathbf{z}}$. Each contribution term consists of the potential function 
    $\Psi(\mathbf{z})$ weighted by $\phi_t(\mathbf{z})$. 
    By considering a single term at specific location $\mathbf{z}$,  viz. $J(\mathbf{z})$, we can
    describe the training dynamics through the continuity equation, namely: 
    \begin{align}
        \frac{\partial \phi_t(\mathbf{z})}{\partial t}=-\nabla_{\mathbf{z}}\cdot\phi_t(\mathbf{z})v(\mathbf{z})
        \label{eq:continuity}
    \end{align}
    where $v(\mathbf{z})=-\nabla_{\mathbf{z}}J(\mathbf{z})$. The equation~(\ref{eq:continuity})
    puts in relation the motion of particle $\mathbf{z}$, moving with speed $v(\mathbf{z})$,
    with the change in $\phi_t(\mathbf{z})$. Therefore,
    \begin{align}
        \frac{\partial \phi_t(\mathbf{z})}{\partial t}&=-\nabla_{\mathbf{z}}{\cdot}\phi_t(\mathbf{z})v(\mathbf{z}) \nonumber\\
                                                    &=\nabla_{\mathbf{z}}{\cdot}\phi_t(\mathbf{z})\nabla_{\mathbf{z}}J(\mathbf{z}) \nonumber\\
                                                    &= \nabla_{\mathbf{z}}\phi_t(\mathbf{z}){\cdot}\nabla_{\mathbf{z}}J(\mathbf{z}) +
                                                        \phi_t(\mathbf{z})\nabla_{\mathbf{z}}{\cdot}\nabla_{\mathbf{z}}J(\mathbf{z}) \nonumber\\
                                                    &= \nabla_{\mathbf{z}}\phi_t(\mathbf{z}){\cdot}\nabla_{\mathbf{z}}J(\mathbf{z}) +
                                                    \phi_t(\mathbf{z})\nabla_{\mathbf{z}}{\cdot}\nabla_{\mathbf{z}}\phi_t(\mathbf{z})\Psi(\mathbf{z}) \nonumber\\
                                                    &= \nabla_{\mathbf{z}}\phi_t(\mathbf{z}){\cdot}\nabla_{\mathbf{z}}J(\mathbf{z}) +
                                                    \phi_t(\mathbf{z})\nabla_{\mathbf{z}}{\cdot}\Big(\nabla_{\mathbf{z}}\phi_t(\mathbf{z})\Big)\Psi(\mathbf{z}) + \nonumber\\
                                                    &\qquad \phi_t(\mathbf{z})\nabla_{\mathbf{z}}{\cdot}\phi_t(\mathbf{z})\nabla_{\mathbf{z}}\Psi(\mathbf{z}) \nonumber\\
                                                    &= \nabla_{\mathbf{z}}\phi_t(\mathbf{z}){\cdot}\nabla_{\mathbf{z}}J(\mathbf{z}) +
                                                    \phi_t(\mathbf{z})\nabla_{\mathbf{z}}{\cdot}\Big(\nabla_{\mathbf{z}}\phi_t(\mathbf{z})\Big)\Psi(\mathbf{z}) + \nonumber\\
                                                    &\qquad \phi_t(\mathbf{z})\nabla_{\mathbf{z}}\phi_t(\mathbf{z}){\cdot}\nabla_{\mathbf{z}}\Psi(\mathbf{z}){+}
                                                    \Big(\phi_t(\mathbf{z})\Big)^2\nabla_{\mathbf{z}}{\cdot}\nabla_{\mathbf{z}}\Psi(\mathbf{z}) \nonumber
    \end{align}
    We can limit our analysis only to maximal points 
    $\mathbf{z}_{max}=\text{arg}\max_{\mathbf{z}}\|\phi_t(\mathbf{z})\|$,
    for which $\nabla_{\mathbf{z}}\phi_t(\mathbf{z}_{max})=0$, in order to
    prove that $\phi_t(\mathbf{z})\rightarrow 0$  for all 
    $\mathbf{z}\in\Omega_{\mathbf{z}}$ as $t\rightarrow\infty$.
    Consequently, previous equation is simplified as follows:
    \begin{align}
        \frac{\partial \phi_t(\mathbf{z}_{max})}{\partial t}&=
                    \Big(\phi_t(\mathbf{z}_{max})\Big)^2\nabla_{\mathbf{z}}{\cdot}\nabla_{\mathbf{z}}\Psi(\mathbf{z}_{max}) \nonumber\\
                &=\Big(\phi_t(\mathbf{z}_{max})\Big)^2\nabla_{\mathbf{z}}^2\Psi(\mathbf{z}_{max}) \nonumber\\
                &=\Big(\phi_t(\mathbf{z}_{max})\Big)^2\nabla_{\mathbf{z}}^2\int_{\Omega_{\mathbf{z}}}\phi_t(\mathbf{z}')k(\mathbf{z}_{max},\mathbf{z}')d\mathbf{z}' \nonumber\\
                &=\Big(\phi_t(\mathbf{z}_{max})\Big)^2\int_{\Omega_{\mathbf{z}}}\phi_t(\mathbf{z}')\nabla_{\mathbf{z}}^2k(\mathbf{z}_{max},\mathbf{z}')d\mathbf{z}' \nonumber\\
                &=-\Big(\phi_t(\mathbf{z}_{max})\Big)^2\int_{\Omega_{\mathbf{z}}}\phi_t(\mathbf{z}')\delta(\mathbf{z}_{max}-\mathbf{z}')d\mathbf{z}' \nonumber\\
                &=-\Big(\phi_t(\mathbf{z}_{max})\Big)^2\int_{\Omega_{\mathbf{z}}}\phi_t(\mathbf{z}')\delta(\mathbf{z}'-\mathbf{z}_{max})d\mathbf{z}' \nonumber\\
                &=-\Big(\phi_t(\mathbf{z}_{max})\Big)^2\phi_t(\mathbf{z}_{max}) \nonumber\\
                &=-\Big(\phi_t(\mathbf{z}_{max})\Big)^3 \nonumber
    \end{align}
    The solution of this differential equation is given by:
    \begin{align}
        \phi_t(\mathbf{z}_{max})=\pm\frac{1}{\sqrt{2t+c}}
    \end{align}
    for a given $c\in\mathbb{R}^+$. Note that
    $\phi_0(\mathbf{z}_{max})=\pm\frac{1}{\sqrt{c}}$. Therefore, $c=(\phi_0(\mathbf{z}_{max}))^{-2}$.
    At the end of training, namely $t\rightarrow\infty$, $\phi(\mathbf{z})=0$ for all
    $\mathbf{z}\in\Omega_\mathbf{z}$, consequently,
    $p_{\mathbf{Z}}(\mathbf{z})=q_{\mathbf{Z}}(\mathbf{z})$ and
    the objective is at its global minimum. This concludes the proof.
\end{proof}

%
%
\section{Solution of the Poisson's equation}
\begin{proof}
    It is important to mention that the solution of the Poisson equation 
    is an already known mathematical result. Nonetheless, we provide here 
    its derivation, since we believe that this can provide useful support for
    the reading of the article.
    Recall that for a given $\mathbf{z}'\in\mathbb{R}^h$
    \begin{equation}
        -\nabla_\mathbf{z}k(\mathbf{z},\mathbf{z}')=\delta(\mathbf{z}-\mathbf{z}'),\quad 
        \forall \mathbf{z}\in\mathbb{R}^h
        \label{eq:poisson}
    \end{equation}
    is the Poisson equation. Note that we are looking for kernel functions 
    that are translation invariant, namely satisfying the property 
    $k(\mathbf{z},\mathbf{z}')=\bar{k}(\mathbf{z}-\mathbf{z}')$. Therefore, 
    the solution of~(\ref{eq:poisson}) can be obtained (i) by considering
    the simplified case where $\mathbf{z}'=\mathbf{0}$ and then (ii) by
    replacing $\mathbf{z}$ with $\mathbf{z}-\mathbf{z}'$ to get the general
    solution.

    Therefore, our aim is to derive the solution for the following case:
    \begin{equation}
        \nabla_\mathbf{z}\Gamma(\mathbf{z})=\delta(\mathbf{z}),\quad 
        \forall \mathbf{z}\in\mathbb{R}^h
        \label{eq:poisson2}
    \end{equation}
    where $\Gamma(\mathbf{z})\doteq-\bar{k}(\mathbf{z})$.

    Consider that $\forall\mathbf{z}\neq\mathbf{0}$ (\ref{eq:poisson2})
    is equivalent to
    \begin{equation}
        \nabla_\mathbf{z}\Gamma(\mathbf{z})=0
        \label{eq:homo}
    \end{equation}
    Now assume that $\Gamma(\mathbf{z})=v(r)$ for some function $v:\mathbb{R}\rightarrow\mathbb{R}$
    and $r\doteq\|\mathbf{z}\|$. Then, we have that $\forall i=1,\dots,h$
    \begin{align}
        \frac{\partial\Gamma(\mathbf{z})}{\partial z_i} &= \frac{dv(r)}{dr}\frac{\partial r}{\partial z_i}=v'(r)\frac{z_i}{r}  \nonumber\\
        \frac{\partial^2\Gamma(\mathbf{z})}{\partial z_i^2} &= v^{''}(r)\frac{z_i^2}{r^2} + v'(r)\frac{1}{r} - v'(r)\frac{z_i^2}{r^3}
        \label{eq:laplacian}
    \end{align}
    By using~(\ref{eq:homo}) and (\ref{eq:laplacian}), we get the following equation:
    \begin{equation}
        \nabla_\mathbf{z}\Gamma(\mathbf{z})=\sum_{i=1}^h\frac{\partial^2\Gamma(\mathbf{z})}{\partial z_i^2}=v^{''}(r)+\frac{h-1}{r}v'(r)=0 \nonumber
    \end{equation}
    whose solution is given by $v'(r)=b/r^{h-1}$ for any scalar $b\neq 0$. By integrating this solution, we obtain
    that:
    \begin{equation}
        v(r)=\left\{\begin{array}{ll}
            br+c & h=1 \\
            b\ln(r)+c & h=2 \\
            -\frac{b}{(h-2)r^{h-2}}+c & h>2
        \end{array}\right.
    \end{equation}
    and by choosing $c=0$ (without loosing in generality), we get that
    \begin{equation}
        \Gamma(\mathbf{z})=\left\{\begin{array}{ll}
            b\|\mathbf{z}\| & h=1 \\
            b\ln(\|\mathbf{z}\|) & h=2 \\
            -\frac{b}{(h-2)\|\mathbf{z}\|^{h-2}} & h>2
        \end{array}\right.
    \end{equation}
    Note that this is the solution of the homogeneous equation in~(\ref{eq:homo}).
    The solution for the nonhomogeneous case in (\ref{eq:poisson2}) can be obtained
    by applying the fundamental theorem of calculus for $h=1$, the Green's theorem
    for $h=2$ and the Stokes' theorem for general $h$ (we skip here the tedious
    derivation, but this result can be easily checked by consulting any book of vector
    calculus for the Green's function). Therefore,
    \begin{equation}
        \Gamma(\mathbf{z})=\left\{\begin{array}{ll}
            \frac{1}{2}\|\mathbf{z}\| & h=1 \\
            \frac{1}{2\pi}\ln(\|\mathbf{z}\|) & h=2 \\
            -\frac{1}{(h-2)\mathcal{S}_{h}\|\mathbf{z}\|^{h-2}} & h>2
        \end{array}\right.
    \end{equation}
    In other words,
    \begin{equation}
        \bar{k}(\mathbf{z})=\left\{\begin{array}{ll}
            -\frac{1}{2}\|\mathbf{z}\| & h=1 \\
            -\frac{1}{2\pi}\ln(\|\mathbf{z}\|) & h=2 \\
            \frac{1}{(h-2)\mathcal{S}_{h}\|\mathbf{z}\|^{h-2}} & h>2
        \end{array}\right.
    \end{equation}
    and after replacing $\mathbf{z}$ with $\mathbf{z}-\mathbf{z}'$, we obtain
    our final result.
\end{proof}

\end{document}